\documentclass[english,11pt]{article}
\include{macros}

\begin{document}

\title{\Large LLM-Powered Virtual Population for Demand Simulation and Pricing}

\author{
	Chengpiao Huang\footnote{Department of Industrial Engineering and Operations Research, Columbia University. Email: \href{mailto:chengpiao.huang@columbia.edu}{chengpiao.huang@columbia.edu}.}
	\and
	Kaizheng Wang\footnote{Department of Industrial Engineering and Operations Research and Data Science Institute, Columbia University. Email: \href{mailto:kaizheng.wang@columbia.edu}{kaizheng.wang@columbia.edu}.}
}

\maketitle

\begin{abstract}

We develop an LLM-powered virtual population model that simulates demand for pricing decisions, in settings where products are described by rich unstructured information, such as text descriptions and images, and where decision makers need not only mean-demand predictions but also uncertainty estimates for counterfactual prices. Our model represents exposed customers as draws from a finite mixture of customer personas. For each persona, product, and candidate price, an LLM elicits a persona-level purchase probability using both structured persona information and unstructured product information. These probabilities are aggregated through calibrated mixture weights to form a predictive distribution of aggregate demand. The resulting simulator can evaluate counterfactual prices under various pricing objectives, including expected revenue and risk-aware criteria such as conditional value at risk. We test the framework on an online H\&M fashion dataset with product descriptions and images. The calibrated LLM-based simulator achieves the best overall predictive performance among the models considered, and supports sample-efficient pricing decisions. Our framework provides a practical way to use LLMs as demand simulators for products with limited historical demand data but rich product information. By producing a full predictive demand distribution rather than only a point forecast, it enables managers to compare candidate prices, quantify demand uncertainty, and choose prices that target either average-case revenue or risk-aware objectives.

\end{abstract}


\section{Introduction}

Demand modeling is a central problem in pricing. Classical approaches typically rely on explicitly specified parametric (e.g., linear) demand models \citep{BRu12, KZe14, CLP20, BKe21}. While these models offer interpretability and strong theoretical foundations, they impose restrictive assumptions on functional forms, and often struggle to capture rich product information such as detailed text descriptions and images.

Recent advances in LLMs offer a new opportunity to address this challenge. Unlike traditional parametric demand models, LLMs can directly process rich, unstructured product representations, including images and text descriptions of varying lengths \citep{BHA21}. This flexibility allows demand responses to depend on information that is difficult to encode in fixed-dimensional feature vectors. As a result, LLMs offer a different modeling paradigm: instead of manually specifying a demand function over a pre-defined set of covariates, we delegate feature construction and modeling to the LLM.

In addition, LLMs provide a flexible framework for risk-aware decision-making. Traditional demand models typically focus on estimating mean demand and therefore lead naturally to objectives such as the expected revenue. In practice, however, decision makers often care not only about average performance but also about the uncertainty and risk of pricing decisions. As LLMs naturally generate probability distributions, they support modeling the entire demand distribution rather than only its mean. This makes it possible to evaluate and optimize risk-aware objectives, such as conditional value at risk (CVaR), in addition to the standard expected revenue.

In this paper, we propose an LLM-based framework for demand simulation in pricing that incorporates rich unstructured product information, including text and images. The key idea is to elicit persona-level purchase probabilities from an LLM and aggregate them into a demand distribution. Our simulator allows for counterfactual evaluation of prices for new products. As our framework models the entire demand distribution, it supports not only the standard expected revenue objective, but also risk-aware pricing objectives such as CVaR. Using an online sales dataset from H\&M, we show that the fitted simulator is well-aligned with the observed data.

\paragraph{Related works.} Our work lies at the intersection of pricing, LLM-based simulation, and risk-aware decision-making. Below we give a review of the related works, which is by no means exhaustive. 

Pricing and demand estimation are central topics in operations research. Classical models typically impose a structured demand model, such as a parametric and feature-based model, and then design pricing policies that balance learning and revenue maximization \citep{BRu12,KZe14,CLP20,BKe21}. While these models offer interpretability, tractable optimization, and strong theoretical foundations, they impose restrictive assumptions on functional forms, and often struggle to capture rich product information such as detailed text descriptions and images.

A growing literature uses LLMs to simulate human behavior in various areas, including economic and social science experiments \citep{AAK23,CLS23,Hor23,ZHS24} and market research \citep{BIN23,GTo23,GSi24}. This line of work suggests that LLMs can serve as flexible simulators of heterogeneous human responses, while also emphasizing the need for validation and calibration against observed data. Motivated by this, a number of related works study how LLM-generated synthetic data should be calibrated \citep{LSA24,BNS25,WZZ26,HWW25,YXi26,FHP26,LWZ26} so that synthetic response distributions better align with real populations. Our work has a similar spirit, where we calibrate and aggregate LLM-simulated outputs to construct a demand distribution for pricing.

A number of recent works have also used LLMs or AI agents in operations and supply chains. One line of work uses LLMs or transformer directly as decision-makers for inventory control, supply-chain optimization, and other operational tasks \citep{CLR21,LMZ23,DHJ25,LCZ25,WCT25,BFM26}. Another line uses LLMs to simulate operational environments or generate synthetic data for downstream decisions \citep{BCC26,WZh26}. Our work belongs to the latter: we use LLMs to simulate customer purchase behavior and fit a demand distribution, and the final pricing decision is made by optimizing an objective over the fitted distribution.

Our work is also related to the literature on risk-aware decision-making, which accounts for adverse outcomes beyond average performance \citep{SDR21}. A large stream of work studies coherent risk measures and tail-risk objectives, including CVaR, which admits tractable reformulations and has become a standard criterion for risk-aware optimization \citep{ADE99,RUr00}. Recent data-driven work further studies risk-aware decision-making with contextual information, sample-based approximations, profit-risk constraints, and distributionally robust shortfall-risk formulations \citep{LCL22,GXu18,TDX25}. A complementary line connects uncertainty quantification to downstream decisions, including conformal risk control and prediction-set-based decision rules that optimize robust risk under coverage guarantees \citep{ABF24,WDo26}. In contrast to works that train a model for one prespecified risk criterion or shifted scenario distribution, our framework directly models a full predictive demand distribution, which can then be reused for multiple pricing objectives.
Recent works also apply generative models, such as diffusion, flow-based, and decision-focused generative models, to scenario generation and robust or risk-sensitive decision-making \citep{WCL26,CZX26}.

\section{Problem Setup and Methodology}

\subsection{Background}

Our goal is to build a simulator for customer demand that can be used to design and evaluate downstream decision-making algorithms, such as pricing policies. Given a product $u$, rich product information (such as detailed text descriptions and images), and a posted price $p$, the simulator specifies a demand distribution from which demand samples can be generated.

To build this simulator, we assume access to historical data on realized demand under different prices. Let $\{u_j:j\in\cJ\}$ denote the products. For each product $j\in\cJ$, we observe rich product information, together with historical price-demand pairs $\cD_j = \{(p_{j,i},d_{j,i}):i\in\cI_j\}$, where $p_{j,i}$ is the posted price for product $j$ in instance $i$ (e.g., on a day), and $d_{j,i}$ is the corresponding demand.

Our goal is to use the dataset $\bigcup_{j\in\cJ}\cD_j$ to construct a demand simulator that generalizes across products and prices while conditioning on rich unstructured product information. Traditional parametric demand models, which rely on fixed-dimensional feature representations and pre-specified functional forms, are not well suited to this setting. In particular, they cannot naturally accommodate variable-length text descriptions or image inputs.

To handle rich product information such as text descriptions and images, we use an LLM to simulate purchase decisions for heterogeneous customer groups, then aggregate these simulated decisions into an aggregate demand model. This approach captures customer heterogeneity without imposing restrictive parametric assumptions on preferences or price sensitivity. We now introduce our model.

\subsection{Persona-Based Modeling} 

\paragraph{A virtual population.}
Consider an e-commerce platform on which a product $u$ is displayed at a posted price $p$. On a given day, suppose that $N$ customers are exposed to the product information through search, browsing, or recommendation. We view these customers as independent draws from a virtual population. With probability $\pi_0$, a customer is not actively interested in the product and therefore does not purchase, regardless of price. Otherwise, the customer belongs to one of $K$ interested customer types, described by profiles $\{x_k\}_{k=1}^K$, with population proportions $\{\pi_k\}_{k=1}^K$. Conditional on being an interested type-$k$ customer, the customer purchases one unit with probability $q(x_k,u,p)$.

Under this construction, the marginal purchase probability of an exposed customer is
\[
\sum_{k=1}^K \pi_k(1-\pi_0)q(x_k,u,p).
\]
Equivalently, if we define $\alpha_k=\pi_k(1-\pi_0)$ for $k\in[K]$ and $\alpha_{K+1}=\pi_0$, then $\balpha\in\Delta^K$ and the virtual population can be represented as a mixture of $K$ active personas plus one inactive dummy persona. \Cref{fig-virtual-population} illustrates this aggregation logic.

\begin{figure}[h]
	\centering
	\caption{Virtual Population Induced by LLM Personas\label{fig-virtual-population}}
	\includegraphics[width=0.8\linewidth]{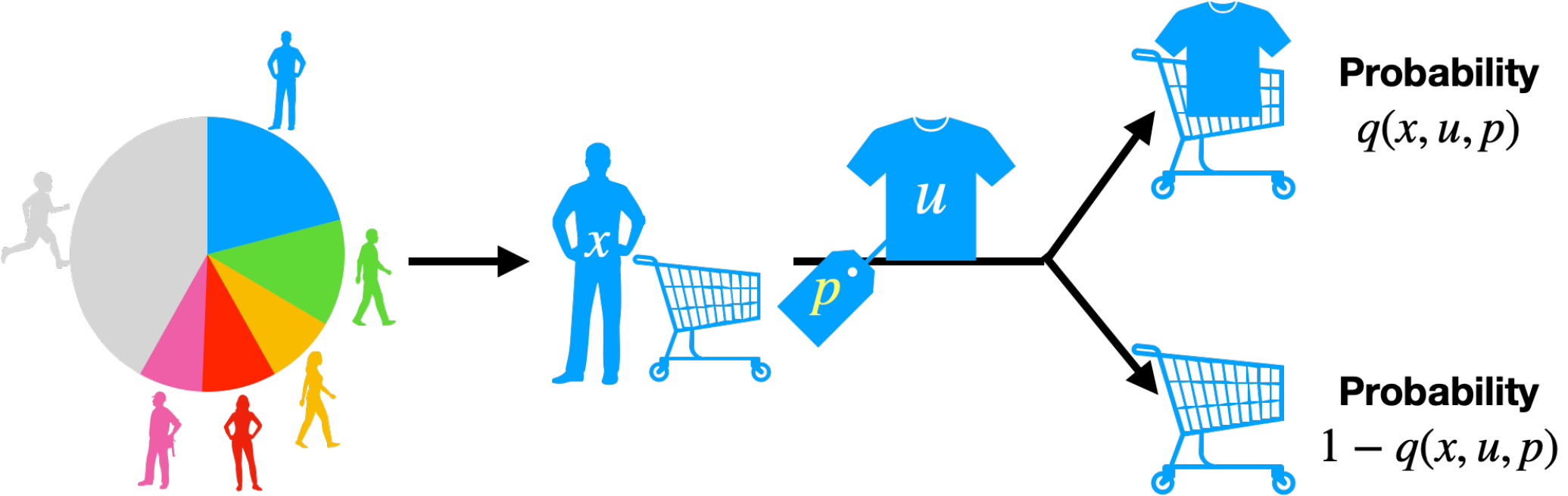}
\end{figure}

We now describe how we use an LLM to simulate this virtual population. We first construct a collection of customer personas $x_1,...,x_K$. Each persona represents a behavioral archetype characterized by demographic and purchasing attributes, such as age group, engagement frequency, typical paid price, and top purchase product category. For each persona $x_k$, product $u$, and candidate price $p$, we query the LLM for a purchase probability. The prompt includes persona characteristics, product information (text and image), and price $p$, and asks for the likelihood $q(x_k,u,p)\in [0,1]$ that persona $x_k$ purchases product $u$ at price $p$.

From a modeling perspective, the LLM-generated purchase probabilities $\{q(x_k,u,p)\}_{k=1}^K$ can be interpreted as \emph{persona- and price-dependent behavioral features}. Unlike traditional parametric demand models, which rely on fixed-dimensional product covariates and impose prespecified functional forms over price, the LLM implicitly maps rich product information, customer attributes, and price into a feature space. This feature space varies across personas and prices, and can flexibly adapt to unstructured product information.

To aggregate demand across heterogeneous customer personas $x_1,...,x_K$, we assume that exposed customers make independent binary purchase decisions. For a given product $u$ and price $p$, we model each customer's purchase probability as a convex combination of persona-level probabilities:
\[
q_{\balpha}(u,p) = \sum_{k=1}^K \alpha_k  q(x_k,u,p) + \alpha_{K+1}\cdot 0,
\]
where $\balpha = (\alpha_1,\ldots,\alpha_K,\alpha_{K+1}) \in \Delta^K$ denotes the mixture weights. The $(K+1)$-th persona is a dummy persona that never buys, which captures baseline non-purchase behavior. The weights $\balpha$ represent the relative contribution of each persona to the aggregate demand and will be estimated from historical data. Under this model, if $N$ customers are exposed to product $u$ at price $p$, then the aggregate demand follows the distribution 
\begin{equation}\label{eqn-llm-mix-binom-model}
	\cQ_{N,\balpha}(u,p) = \binomial\left(N,\, q_{\balpha}(u,p)\right).
\end{equation}

\paragraph{Monotone calibration.} The LLM-simulated purchase probabilities $q(x_k,u,p)$ may be systematically biased even when they contain useful ordinal information. Prior works have shown that language-model confidence scores and verbalized probabilities can be miscalibrated, although they often retain useful ranking information and can be improved by elicitation or post-hoc calibration \citep{JAD21,TMS23,PKT24,WSB24,KGR24}. To adjust for such systematic bias, we also consider a monotone calibrated version of the persona-level purchase probabilities. Let $\sigma(z) = 1/(1+e^{-z})$ denote the sigmoid function, which monotonically transforms a real number to a number in $[0,1]$. Its inverse function is $\sigma^{-1}(q) = \log( \frac{q}{1-q})$. For parameters $a\in\RR$ and $b>0$, define the monotone transformation
\[
T_{a,b}(q)=\sigma\left(a+b\cdot\sigma^{-1}(q)\right), \qquad\forall q\in[0,1].
\]
The parameter $a$ shifts the overall level of purchase probability, while $b$ controls the strength of the LLM's implied probability scale. The restriction $b>0$ preserves the ordering of purchase probabilities across products and personas. The calibrated aggregate purchase probability is
\[
q_{\balpha,a,b}^{\calb}(u,p)
=
\sum_{k=1}^K \alpha_k T_{a,b}\big(q(x_k,u,p)\big)
+ \alpha_{K+1}\cdot 0,
\]
and the corresponding calibrated demand model is
\begin{equation}\label{eqn-llm-mix-binom-model-cal}
	\cQ_{N,\balpha,a,b}^{\calb}(u,p)
	=
	\binomial\left(N,\, q_{\balpha,a,b}^{\calb}(u,p)\right).
\end{equation}
The original model $\cQ_{N,\balpha}(u,p)$ in \eqref{eqn-llm-mix-binom-model} is recovered by setting $a=0$ and $b=1$. Intuitively, the calibration operates on the log-odds scale. The intercept $a$ captures a global level bias in the LLM probabilities: a negative value lowers all active-persona purchase probabilities, while a positive value raises them. The slope $b$ acts like an inverse-temperature parameter. When $b<1$, extreme LLM probabilities are shrunk toward $1/2$; when $b>1$, differences between low and high LLM probabilities are amplified. The constraint $b>0$ preserves the LLM-implied ranking across personas, products, and prices, so calibration corrects the scale of the probabilities without discarding their ordinal information.

\subsection{Persona Weights Estimation} 

We estimate the persona mixture weights $\balpha$ and the exposure parameter $N$ from the historical dataset $\cD$ by maximum likelihood. A natural formulation is
\begin{equation}\label{eqn-obj-MLE}
	\min_{\balpha \in \Delta^{K}, \, N \in \mathbb{Z}_+}
	\quad
	- \sum_{j\in\cJ} \sum_{i\in\cI_j} \log \ell \left( d_{j,i} \bigm|  N, \, q_{\balpha}( u_j,p_{j,i})  \right),
\end{equation}
where $\ell(d \mid N, q) = \PP\left( z = d \right)$, with $z\sim \binomial(N,q)$, is the likelihood function for the binomial distribution.

In practice, however, the data may be \emph{truncated}: we only observe product-price pairs with positive realized demand. This is indeed the case in our dataset \citep{HM22}. To account for this, we instead solve
\begin{equation}\label{eqn-obj-trunc-raw}
	\min_{\balpha \in \Delta^{K}, \, N \in \mathbb{Z}_+}
	\quad
	-\sum_{j\in\cJ} \sum_{i\in\cI_j} \log \widetilde{\ell} \left( d_{j,i} \bigm|  N, \, q_{\balpha}( u_j,p_{j,i})  \right),
\end{equation}
where $\widetilde{\ell}(d \mid N, q) = \PP\left( z= d \mid z>0 \right)$ with $z\sim \binomial(N,q)$. Thus, $\widetilde{\ell}$ is the likelihood conditional on observing strictly positive demand. \Cref{lem-obj-equiv} below shows that this conditional likelihood admits an equivalent tractable form.

\begin{lemma}[Equivalent truncated objective]\label{lem-obj-equiv}
	Fix $q\in(0,1]$, $N\in\ZZ_+$ and $d\in[N]$. Define
	\begin{equation}\label{eqn-trunc-obj-clean}
		f(d\mid N,q) =  - d\log q - (N-d) \log (1-q) + \log \left[ 1 - (1-q)^N \right].
	\end{equation}
	For $z\sim \binomial(N,q)$, it holds that
	\[
	-\log \PP\left( z= d \mid z>0 \right)
	=
	f(d\mid N,q) - \log \binom{N}{d}.
	\]
	Moreover, $q\mapsto f(d\mid N,q)$ is a convex function for $q\in(0,1/2]$.
\end{lemma}

\begin{proof}[Proof of \Cref{lem-obj-equiv}]
	See \Cref{sec-lem-obj-equiv-proof}.
\end{proof}

By \Cref{lem-obj-equiv}, and because $\balpha\mapsto q_{\balpha}( u,p)$ is linear, the following optimization problem is convex for each fixed $N$:
\begin{align}
	\min_{\balpha}
	&\quad
	\sum_{j\in\cJ} \sum_{i\in\cI_j} \left[ f\big (d_{j,i} \bigm| N,q_{\balpha}(u_j,p_{j,i}) \big) - \log\binom{N}{d_{j,i}} \right]  \notag \\[4pt]
	\text{s.t.}&\quad q_{\balpha}(u_j,p_{j,i}) \le \frac{1}{2}, \qquad \forall j\in\cJ,\quad \forall i\in\cI_j, \label{eqn-optimization-fix-N} \\[4pt]
	&\quad \balpha\in\Delta^{K}. \notag
\end{align}
Here, the constraint $q_{\balpha}(u_j,p_{j,i}) \le \frac{1}{2}$ guarantees the convexity of the problem. It is motivated by the fact that most customers do not purchase, and thus we expect these purchase probabilities to be well below $1/2$ in practical applications. This leads to the estimation procedure in \Cref{alg-weights-opt}.

\begin{algorithm}[h]
	{\bf Input:} Product-price-demand data $\{(u_j,p_{j,i},d_{j,i}):i\in\cI_j,\, j\in\cJ\}$, LLM-simulated purchase probabilities $\{q(x_k,u_j,p_{j,i}):k\in[K],\, i\in\cI_j, \,j\in\cJ\}$, and maximum exposure $N_{\max}$. \\
	{\bf For $N=1,...,N_{\max}$:} \\
	\hspace*{.6cm} If the demand is always observed, solve the optimization problem \eqref{eqn-obj-MLE}; if zero demand is truncated, solve the optimization problem \eqref{eqn-optimization-fix-N}.\\
	\hspace*{.6cm} Get an optimal solution $\widehat{\balpha}^{(N)}$ and optimal value $V^{(N)}$. \\
	Find $\widehat{N} = \argmin_{N\in[N_{\max}]} V^{(N)}$ and set $\widehat{\balpha} = \widehat{\balpha}^{(\widehat{N})}$. \\
	{\bf Output:} $(\widehat{N},\widehat{\balpha})$.\caption{Persona Weights and Exposure Optimization}
	\label{alg-weights-opt}
\end{algorithm}

For the calibrated model $\cQ_{N,\balpha,a,b}^{\calb}$, we estimate $(N,\balpha,a,b)$ by the same likelihood principle, replacing $q_{\balpha}(u,p)$ in \eqref{eqn-obj-MLE} or \eqref{eqn-optimization-fix-N} with $q_{\balpha,a,b}^{\calb}(u,p)$. For fixed $(a,b,N)$, optimization over $\balpha$ has the same form as above after replacing the purchase probabilities by their calibrated version. In implementation, we alternate between fitting the mixture weights $\balpha$ and updating the calibration parameters $(a,b)$.

\subsection{Demand Simulation and Pricing} 

Once the unknown parameters are estimated, the demand simulator is fully specified and can generate demand for new products under counterfactual prices. Given a product $u$ and candidate price $p$, we first query the LLM for persona-level purchase probabilities $\{q(x_k,u,p)\}_{k=1}^K$. We then combine these probabilities using the estimated mixture weights and sample synthetic demand from
\[ 
\widehat{\syndist}(u,p)
=
\cQ_{\widehat{N},\widehat{\balpha}}(u,p)
=
\binomial\big( \widehat{N} , \, q_{\widehat{\balpha}} (u,p) \big),
\]
or 
\[
\widehat{\syndist}(u,p)
=
\cQ_{\widehat N, \widehat \balpha, \widehat a, \widehat b}^{\calb}(u,p)
=
\binomial\left( \widehat N,\, q_{ \widehat \balpha, \widehat a, \widehat b}^{\calb}(u,p)\right).
\]
Repeating this procedure over a grid of prices yields a synthetic demand curve that can be used to evaluate downstream decision-making algorithms, such as dynamic pricing policies, for new products and counterfactual scenarios.

The simulator also induces a pricing rule. Let $F$ be a functional that maps a demand distribution to a scalar objective, such as expected revenue or a risk-aware criterion. Given a product $u$, we choose
\[
\widehat{p} = \argmax_{p} F\Big( \widehat{\syndist}(u,p) \Big).
\]
We now give two examples of $F$.

\begin{example}[Expected revenue]
	The expected revenue is defined by
	\[
	F\big( \cQ(u,p) \big) = p\cdot \EE_{d\sim\cQ(u,p)}\left[ d \right],
	\]
\end{example}

\begin{example}[Conditional value at risk]
	The conditional value at risk (CVaR) of the revenue is defined as follows. Let $d\sim\cQ(u,p)$, then the revenue is $R = p\cdot d$. Given a tail probability parameter $\tau\in(0,1)$, we define $\cvar$ by
	\[
	\cvar_\tau(R)
	=
	\EE\left[ R \mid R \le \quantile_{\tau} (R) \right],
	\]
	which measures the average revenue in the worst $\tau$ fraction of outcomes.
\end{example}

\section{Numerical Experiments}

In this section, we test our framework on real data. The code is available at \url{https://github.com/ch3702/LLM-demand-simulator}.

\subsection{Dataset}

We use an H\&M Personalized Fashion Recommendations dataset \citep{HM22}. The dataset covers customer transactions from September 2018 to September 2020, and consists of the following three types of information. First, transaction records provide the purchase date, customer ID, product ID, price, and sales channel. Second, product metadata includes product IDs, names, types, text descriptions, and images. Third, customer records contain demographic and behavioral attributes such as age and fashion-news engagement.

\paragraph{Products.}
We focus on products in the trousers category sold through the online channel, in the one-year span of September 20, 2018 to September 19, 2019. To obtain a set of products with nontrivial price variation, we restrict attention to items observed at at least five distinct prices and rank products by transaction volume. We take the top $100$ products. For each product, we aggregate transactions to the $(\text{date},\,\text{product},\,\text{price})$ level and define daily demand as the number of purchases observed for that triple.\footnote{In the dataset, the prices are normalized to lie in $[0,1]$. Following the suggestion of \url{https://www.kaggle.com/c/h-and-m-personalized-fashion-recommendations/discussion/310496}, we rescale the prices by $590$.}

\paragraph{Personas.}
To represent heterogeneous customer preferences, we partition historical customers into cells using age, purchase frequency, typical paid price, and dominant product category. Each cell is treated as a persona. We keep the top $K=50$ personas with the most customers, and add a dummy persona that never makes a purchase.

\subsection{LLM Simulation Procedure}

We use GPT-5-mini \citep{GPT5} to simulate customers. We query it with prompts that consist of (i) persona description, (ii) product name, text descriptions and product image, and (iii) candidate prices. The LLM is asked to return persona-specific purchase probabilities at each candidate price. An illustrative prompt is shown in \Cref{fig-prompt}.
\begin{figure}[h]
	\FIGURE{
		\fbox{
			\begin{minipage}{0.95\textwidth}
				\ttfamily
				You are a customer of H\&M in the years 2018-2019. Your age is 25-34. You make about 1 to 6 purchases per month. Your typical paid price is between 18.98 and 22.93. Most of your purchases are in: Trousers. \\
\,\\
				Task: Given a product and a list of prices, return the probability you would buy at each price. Output JSON exactly: \{"prices": [...], "p\_buy": [...], "reason": "<=30 words"\}. \\
\,\\
				Product name: Julia RW Skinny Denim TRS\\
				Product type: Trousers\\
				Product color: Light Blue\\
				Detailed description: 5-pocket jeans in superstretch washed denim with a regular waist and super-skinny legs. \\
				\,\\
				Offered prices (USD): [11.99, 12.99, 13.99, 14.99, 15.99, 16.14, 16.66, 16.99, 17.99, 19.99]
			\end{minipage}
		}
	}
	{Illustrative Prompt for Eliciting Persona-Level Purchase Probabilities\label{fig-prompt}}
	{}
\end{figure}
\Cref{fig-image} shows the associated product image.
\begin{figure}[h]
	\centering
	\caption{Illustrative Product Image Provided to the LLM\label{fig-image}}
	\includegraphics[scale=0.1]{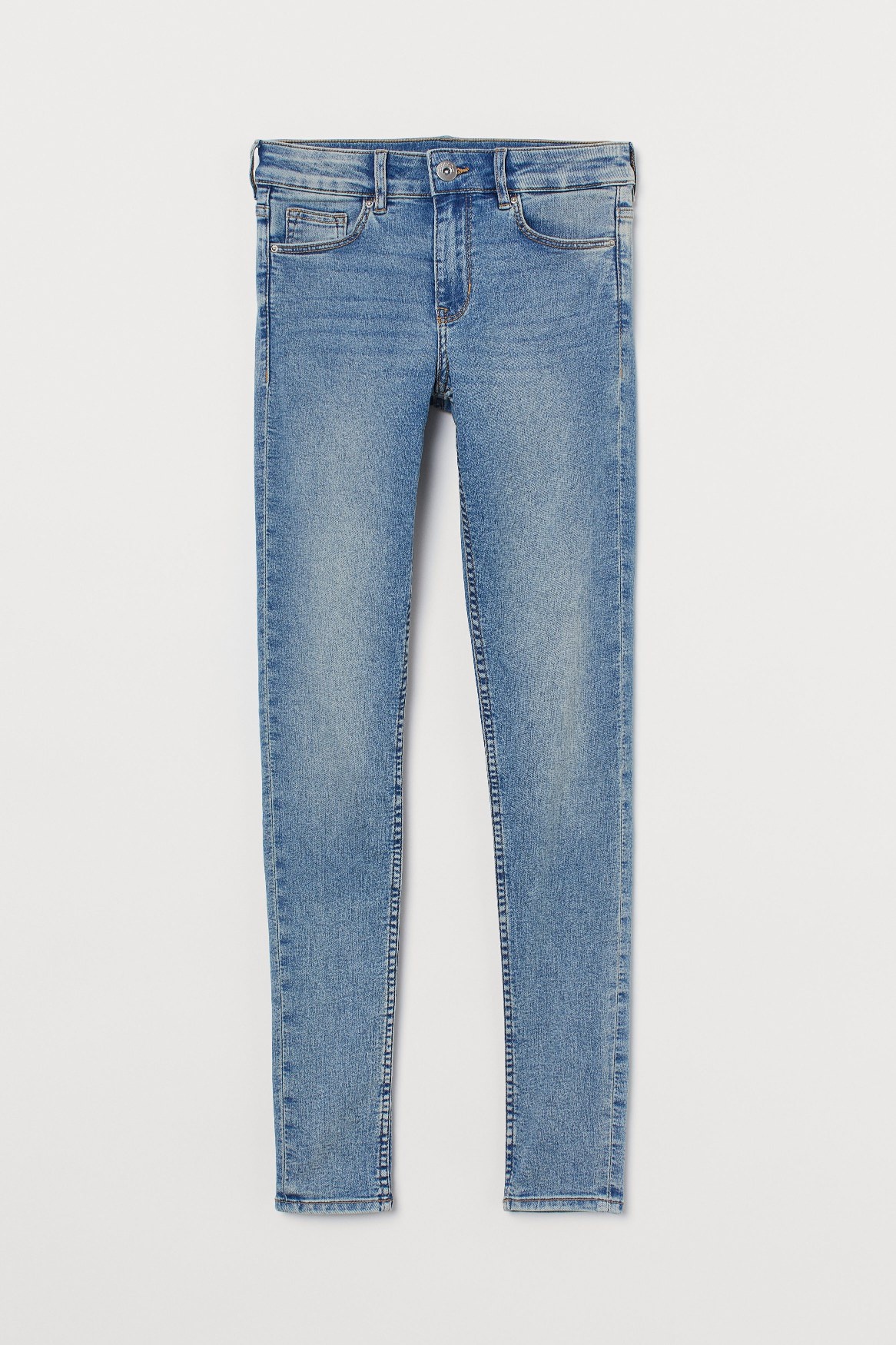}
\end{figure}
A typical LLM response is given in \Cref{fig-response}.
\begin{figure}[h]
	\FIGURE{
		\fbox{
			\begin{minipage}{0.95\textwidth}
				\ttfamily
				\{"prices":[11.99, 12.99, 13.99, 14.99, 15.99, 16.14, 16.66, 16.99, 17.99, 19.99],\\
				"p\_buy":[0.95, 0.93, 0.90, 0.86, 0.80, 0.78, 0.75, 0.72, 0.68, 0.64],\\
				"reason":"I'm a frequent H\&M trousers buyer who usually pays ~\$19-23; discounts sharply raise likelihood, decreasing as price approaches my usual spend."\}
			\end{minipage}
		}
	}
	{Illustrative LLM Response\label{fig-response}}
	{}
\end{figure}

\subsection{Illustrative Use Case}

We now present a simple use case to illustrate how a fitted LLM-based simulator is used for counterfactual demand simulation and pricing. Consider one H\&M trousers product in our study; its product image and the prompt used to elicit persona-level purchase probabilities are shown later in \Cref{fig-image,fig-prompt}. Once the simulator is fitted, it can generate a predicted demand distribution for any candidate counterfactual price.

\begin{figure}[h]
	\FIGURE{
		\subcaptionbox{Price \$11.99}
		{\includegraphics[width=0.32\linewidth]{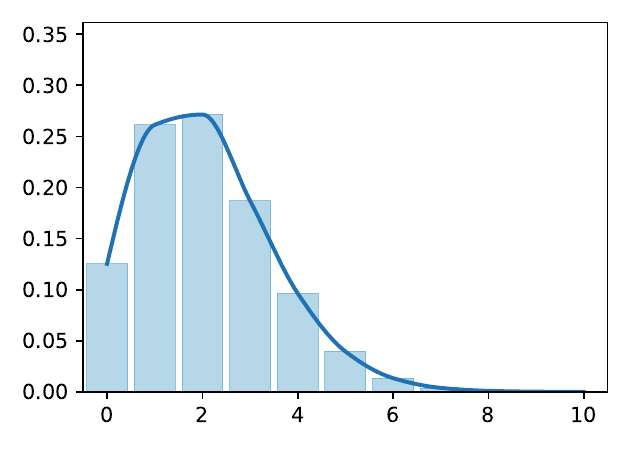}}
		\subcaptionbox{Price \$16.14}
		{\includegraphics[width=0.32\linewidth]{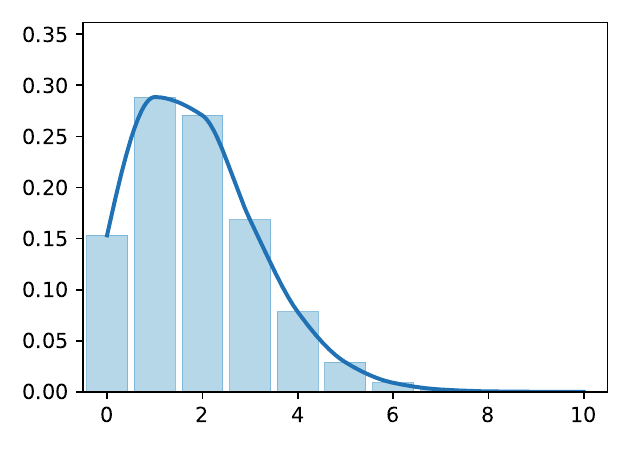}}
		\subcaptionbox{Price \$19.99}
		{\includegraphics[width=0.32\linewidth]{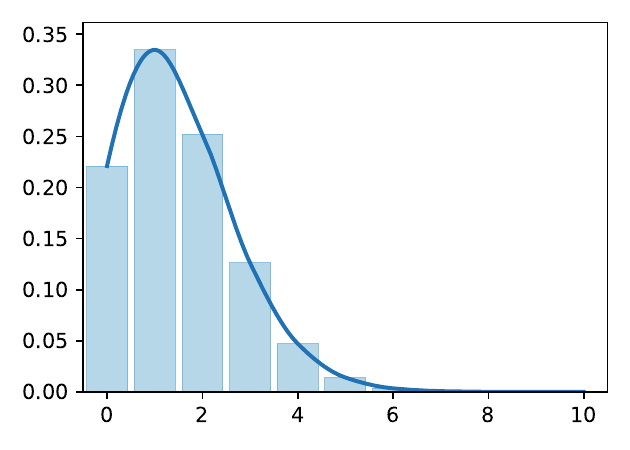}}
	}
	{Simulated Demand Distribution under Different Prices\label{fig-demo-demand}}
	{}
\end{figure}

As shown in \Cref{fig-demo-demand}, the simulator describes how the entire predictive demand distribution shifts with price, rather than only how mean demand changes. This distributional output can then be mapped to different pricing objectives. In \Cref{fig-demo-pricing}, we compare two examples: expected revenue and the $0.25$-level CVaR of revenue.

\begin{figure}[h]
	\FIGURE{
		\subcaptionbox{Expected Revenue}
		{\includegraphics[width=0.48\linewidth]{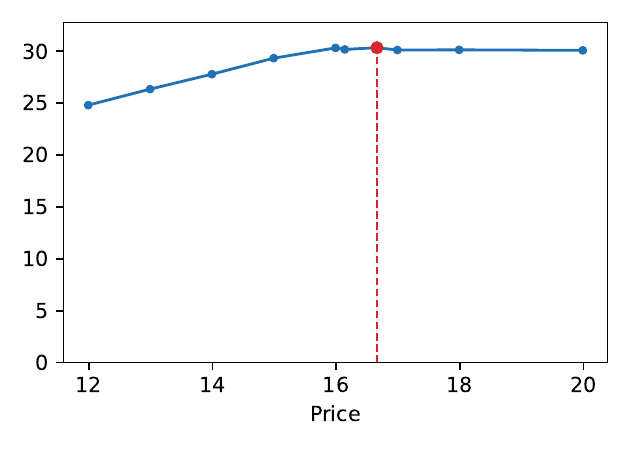}}
		\hfill\subcaptionbox{$0.25$-Level $\cvar$ of Revenue}
		{\includegraphics[width=0.48\linewidth]{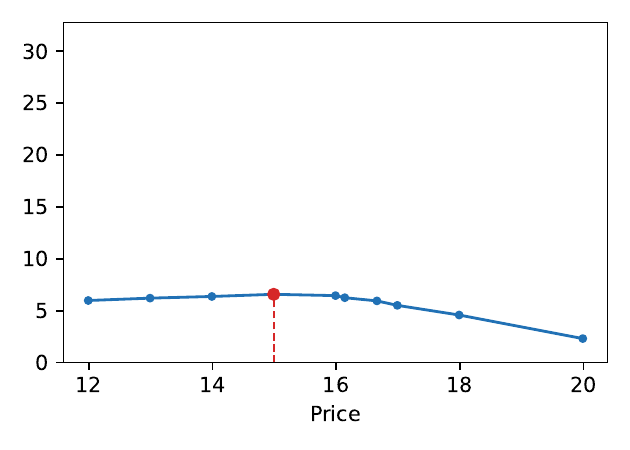}}
	}
	{Pricing Objectives under Counterfactual Prices\label{fig-demo-pricing}}
	{}
\end{figure}

This example highlights the main practical value of the framework. The optimal price depends on the objective: the price that maximizes expected revenue need not be the same as the price that performs best under a risk-aware objective. Because the simulator generates a full predictive distribution, it supports both average-case and risk-aware evaluation of counterfactual prices for new products.

\subsection{Experiments on Demand Prediction}

Next, we will evaluate our models' performance on demand prediction and pricing. We first consider demand prediction. We split the $100$ products $\cJ$ into a training set $\cJ_{\tr}$ and a testing set $\cJ_{\te}$, with $|\cJ_{\tr}|=60$ and $|\cJ_{\te}|=40$. The training set $\cJ_{\tr}$ is used to fit the demand models \eqref{eqn-llm-mix-binom-model} and \eqref{eqn-llm-mix-binom-model-cal} using the truncated likelihood objective \eqref{eqn-optimization-fix-N}. The models are then evaluated on the testing set $\cJ_{\te}$ by comparing the true demand samples and predicted demand distributions.

\paragraph{Evaluation metrics.} We evaluate the fit of a distributional model using the \emph{continuous ranked probability score} and the \emph{probability integral transform}. 
\begin{itemize}
	\item The continuous ranked probability score (CRPS) is proposed by \cite{MWi76} to measure the quality of a distributional forecast. CRPS of a predicted distribution $\widehat{\cP}$ with respect to a real sample $z\sim\cP$ is defined by
	\[
	d(\widehat{\cP},z) = \EE_{z'\sim\widehat{\cP}}|z'-z| - \frac{1}{2}\EE_{z',z''\sim\widehat{\cP}}|z'-z''|.
	\]
	It has the property that $\EE_{z\sim\cP} \big[ d(\widehat{\cP},z)\big] \ge \EE_{z\sim\cP} \big[ d(\cP,z)\big] $ for all $\widehat{\cP}$, and equality holds if and only if $\widehat{\cP}=\cP$. In our setting, we evaluate a predicted demand distribution $\widehat{\cQ}(u,p)$ with respect to the testing set $\cJ_{\te}$ by computing
	\[
	\CRPS(\widehat{\cQ}) = \frac{1}{\sum_{j\in\cJ_{\te}}|\cI_j|} \sum_{j\in\cJ_{\te}} \sum_{i\in\cI_j} d\big( \widehat{\cQ}^+(u_j,p_{j,i}), \, d_{j,i} \big).
	\]
	Here, $\widehat{\cQ}^+(u,p)$ denotes the zero-truncated version of $\widehat{\cQ}(u,p)$, that is, $\widehat{\cQ}^+(u,p) = \Law(z\mid z>0)$ where $z\sim\widehat{\cQ}(u,p)$. This accounts for the fact that we only observe positive demand samples. Lower CRPS indicates more accurate distribution fit.

	\item We also evaluate distributional fit using the randomized probability integral transform (PIT). Let $\widehat{F}^+_{u,p}$ denote the cumulative distribution function (CDF) of $\widehat{\cQ}^+(u,p)$, that is, $\widehat{F}^+_{u,p}(d) = \PP_{z\sim\widehat{\cQ}^+(u,p)}(z\le d)$. Let $\widehat{F}^+_{u,p}(d-) = \PP_{z\sim\widehat{\cQ}^+(u,p)}(z < d)$. Define
	\[
	Y_{j,i} = \widehat{F}^+_{u_j,p_{j,i}}(d_{j,i}-) + V_{j,i} \left[ \widehat{F}^+_{u_j,p_{j,i}}(d_{j,i}) - \widehat{F}^+_{u_j,p_{j,i}}(d_{j,i}-) \right],
	\]
	where $V_{j,i}\sim\cU(0,1)$ i.i.d. If the predicted distribution is accurate, then we expect $\{Y_{j,i}\}$ to be approximately i.i.d.~$\cU(0,1)$. We summarize the deviation of $\{Y_{j,i}\}$ from $\cU(0,1)$ by the Kolmogorov--Smirnov distance
	\[
	\KSPIT(\widehat{\cQ})
	=
	\sup_{t\in[0,1]}
	\left|
	\frac{1}{\sum_{j\in\cJ_{\te}}|\cI_j|}\sum_{j\in\cJ_{\te}}\sum_{i\in\cI_j} \mathbf{1}\{Y_{j,i}\le t\}-t
	\right|,
	\]
	where smaller values indicate better calibration.
\end{itemize}

Additionally, we report the mean absolute error (MAE) and root mean squared error (RMSE) that measures the accuracy for mean prediction. Let $\widehat{\mu}^+(u,p)$ denote the mean of $\widehat{\cQ}^+(u,p)$. We compute
\begin{align*}
	& \MAE(\widehat{\cQ}) = \frac{1}{\sum_{j\in\cJ_{\te}}|\cI_j|}\sum_{j\in\cJ_{\te}}\sum_{i\in\cI_j} \left| \widehat{\mu}^+(u_j,p_{j,i}) - d_{j,i} \right|, \\[4pt]
	& \RMSE(\widehat{\cQ}) = \sqrt{\frac{1}{\sum_{j\in\cJ_{\te}}|\cI_j|}\sum_{j\in\cJ_{\te}}\sum_{i\in\cI_j} \left( \widehat{\mu}^+(u_j,p_{j,i}) - d_{j,i} \right)^2}. 
\end{align*}

\paragraph{Benchmark models.} We compare the proposed LLM-based simulators against two distributional benchmarks. The first is an embedding-based binomial model. Let $\emb(u)$ denote a product embedding constructed from product text and image information, and let $\emb(x_k)$ denote a persona embedding. These embeddings are constructed using the vision-language encoder model SigLIP 2 \citep{SigLIP2}. After standardizing product embeddings, persona embeddings, and prices, the persona-level purchase probability is modeled as
\[
q_{\btheta}^{\emb}(x_k,u,p)
=
\sigma\left(
\btheta^\top
\begin{pmatrix}
	1 \\ p \\ \emb(u) \\ \emb(x_k)
\end{pmatrix}
\right),
\]
and the aggregate purchase probability is
\[
q_{\balpha,\btheta}^{\emb}(u,p)
=
\sum_{k=1}^K \alpha_k q_{\btheta}^{\emb}(x_k,u,p)+\alpha_{K+1}\cdot 0.
\]
The resulting predictive distribution is 
\begin{equation}\label{eqn-embed-binom-model}
	\cQ^{\emb}_{N,\balpha,\btheta}(u,p)=\binomial(N,q_{\balpha,\btheta}^{\emb}(u,p)).
\end{equation}
When fitting the three binomial models \eqref{eqn-llm-mix-binom-model}, \eqref{eqn-llm-mix-binom-model-cal} and \eqref{eqn-embed-binom-model}, to reduce computational cost, we search over $N\in\{100,150,200,250\}$ in \Cref{alg-weights-opt}.

The second benchmark is a normally distributed demand model based on the product embedding and price, which mainly targets modeling the mean demand and does not incorporate persona information. With the same standardized product embedding $\emb(u)$ and price $p$, we fit
\begin{equation}\label{eqn-normal-model}
	\cQ^{\normal}_{\bbeta,\tau}(u,p) = \cN\left( \bbeta^\top
	\begin{pmatrix}
		1 \\ p \\ \emb(u)
	\end{pmatrix}, ~  \tau^2\right)
\end{equation}
For distributional evaluation, we convert this continuous model into a discrete demand distribution by rounding and truncating below zero. When evaluating mean prediction metrics MAE and RMSE, we use the unrounded mean.

\paragraph{Results for demand prediction.}  In \Cref{tab-dist-fit}, we report the distributional fit metrics $\CRPS$ and $\KS_{\PIT}$, as well as mean prediction metrics $\MAE$ and $\RMSE$, averaged over $10$ random train-test splits. For all metrics, smaller values denote better results. Here $\llm$ and $\llmcal$ refer to the mixture-of-persona model \eqref{eqn-llm-mix-binom-model} and its logit-calibrated version \eqref{eqn-llm-mix-binom-model-cal}, respectively, and $\emb$ and $\normal$ are the benchmark embedding-based model \eqref{eqn-embed-binom-model} and normally distributed model \eqref{eqn-normal-model}, respectively.

\begin{table}[h]
	\TABLE
	{Demand Distributional Fit and Mean Prediction Metrics\label{tab-dist-fit}}
	{\begin{tabular}{ccccc}
			\hline
			& $\CRPS$ & $\KSPIT$ & $\MAE$ & $\RMSE$ \\ \hline
			$\llm$ & $1.00$ & $0.32$ & $1.44$ & $1.87$  \\
			$\llmcal$ & $\mathbf{0.94}$ & $\mathbf{0.30}$ & $\mathbf{1.38}$ & $\mathbf{1.79}$ \\
			$\emb$ & $0.99$ & $0.32$ & $1.44$ & $1.86$ \\
			$\normal$ & $1.19$ & $0.43$ & $1.46$ & $1.86$  \\ \hline
	\end{tabular}}
	{}
\end{table}

We observe that $\llmcal$ has the best overall performance, while $\llm$ and $\emb$ have similar performance. This shows that the ordinal information contained in LLM-elicited purchase probabilities can provide additional benefits over using embeddings only, but to realize them, some calibration, such as the logit calibration used by $\llmcal$, is necessary. Moreover, these three methods consistently outperform $\normal$, demonstrating the advantages of persona-based modeling.

\subsection{Experiments on Pricing}

We now turn to a pricing experiment. In particular, we study the sample efficiency of $\llmcal$ to learn a good pricing policy. As the performance evaluation for pricing requires counterfactual demand distributions at different prices, which are not directly available in the dataset, we fit a demand distribution model from the dataset to serve as the ground truth distribution.  Because $\llmcal$ has the best distributional fit in the previous section, we use it to fit a ground truth distribution.

Specifically, we split the $100$ products $\cJ$ into three subsets $\cJ_1$, $\cJ_2$ and $\cJ_3$, with $|\cJ_1|=60$, $|\cJ_2|=25$ and $|\cJ_3|=15$. The first subset $\cJ_1$ is used to fit a ``ground truth'' $\llmcal$ demand distribution $\widehat{\cQ}^*$. We generate synthetic demand samples from $\widehat{\cQ}^*$ at the observed product-price pairs in $\cJ_2$:
\[
\widetilde{\cD}_j=\{(u_j,p_{j,i},\widetilde d_{j,i}):i\in\cI_j\},\qquad \widetilde d_{j,i}\sim \widehat{\cQ}^*(u_j,p_{j,i}),\quad j\in\cJ_2.
\]
For a training fraction $\rho\in(0,1]$, let $\widetilde{\cS}_2(\rho)$ be a random subset containing a $\rho$ fraction of the synthetic samples in $\bigcup_{j\in\cJ_2}\widetilde{\cD}_j$. Using $\widetilde{\cS}_2(\rho)$, we fit an \emph{estimated simulator}
\[
\widehat{\cQ}_{\rho}(u,p)=\cQ_{\widehat N_{\rho},\widehat{\balpha}_{\rho},\widehat a_{\rho},\widehat b_{\rho}}^{\calb}(u,p),
\]
where estimation uses the standard likelihood \eqref{eqn-obj-MLE}. Varying $\rho$ allows us to evaluate how many synthetic samples from $\cJ_2$ are needed to learn a pricing model with good out-of-sample decision quality.

We then evaluate pricing on the held-out products $\cJ_3$. For each $j\in\cJ_3$, let $\Omega_j=\{p_{j,i}:i\in\cI_j\}$ denote the collection of prices for product $j$ observed in the real data. The optimal price under an objective $F$ is
\[
p_j^*=\argmax_{p\in \Omega_j}\ F\big(\widehat{\cQ}^*(u_j,p)\big),
\]
while the price induced by the estimated simulator $\widehat{\cQ}_{\rho}$ is
\[
p^{\rho}_j = \argmax_{p\in \Omega_j} \ F\big(\widehat{\cQ}_{\rho}(u_j,p)\big).
\]
We measure pricing quality by mean performance ratio under the ground-truth demand model:
\[
\Ratio_F(\rho) = \frac{1}{|\cJ_3|}\sum_{j\in\cJ_3}
\frac{F\big(\widehat{\cQ}^*(u_j,p_j^{\rho})\big)}{F\big(\widehat{\cQ}^*(u_j,p_j^*)\big)}.
\]
We consider expected revenue and $\cvar_{0.25}$ as the pricing objectives. Larger values of $\Ratio_F(\rho)$ indicate that the model trained on $\widetilde{\cS}_2(\rho)$ has learned a pricing policy closer to the ground-truth optimal policy.

\paragraph{Results for pricing.}
In \Cref{fig-pricing-efficiency}, we plot $\Ratio_{F}(\rho)$ as a function of the training fraction $\rho\in\{0.01,0.025,0.05\}\cup\{0.1\cdot m:m\in[8]\}$, averaged over $10$ random three-way splits of data. We observe that for $\rho = 0.025$, the learned pricing policy already achieves a relative performance of over $90\%$ for the expected revenue, and $87\%$ for the $0.25$-level CVaR. On average the dataset $\widetilde{\cS}_2$ with $\rho=1$ has approximately $2923$ samples and $25$ products in total, so $\rho=0.025$ translates to $73$ samples, that is, less than $3$ samples per product on average. This shows that our modeling framework is highly sample efficient.

\begin{figure}[h]
	\FIGURE{
		\subcaptionbox{Expected Revenue}
		{\includegraphics[width=0.48\linewidth]{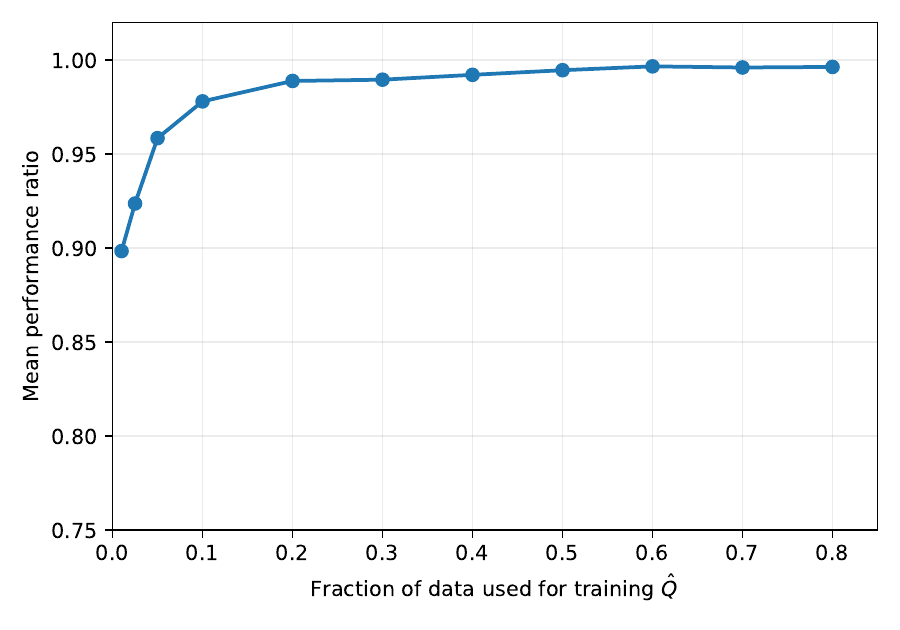}}
		\hfill\subcaptionbox{$0.25$-Level $\cvar$ of Revenue}
		{\includegraphics[width=0.48\linewidth]{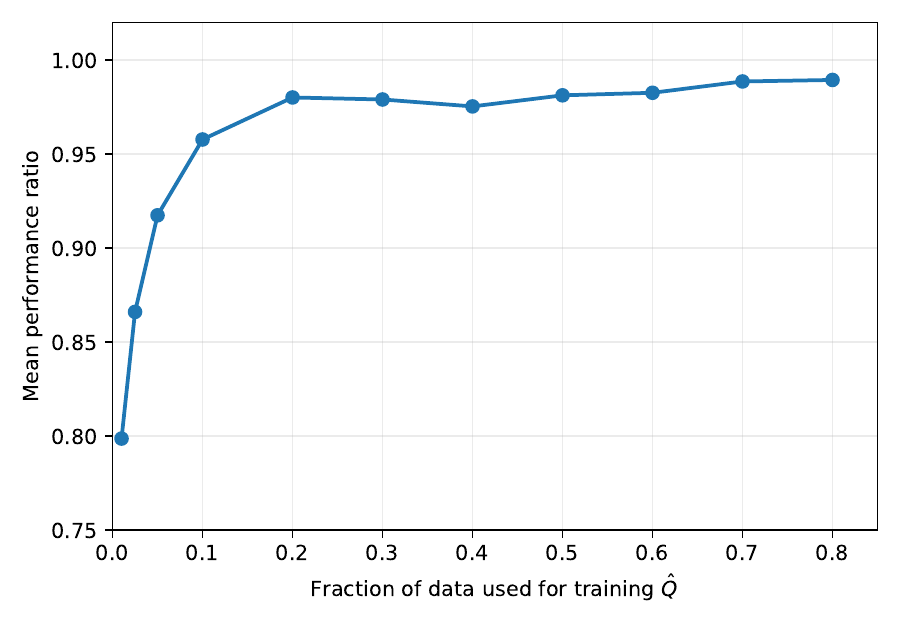}}
	}
	{Pricing Sample Efficiency of the $\llmcal$ Model\label{fig-pricing-efficiency}}
	{}
\end{figure}

\section{Discussions}

This paper develops an LLM-based demand simulator for pricing that can incorporate rich product information including text descriptions and images. The framework aggregates LLM-generated persona-level purchase probabilities into a tractable stochastic demand model that is useful for downstream pricing decisions. Numerical experiments show that the fitted simulator is reasonably well aligned with the observed data and gives strong pricing performance under both expected-revenue and CVaR objectives.

\section*{Acknolwedgement}

This research is supported by the 2025--2026 Deming Doctoral Fellowship from the Deming Center for Operational Innovation and Excellence at Columbia Business School.

\newpage

\appendix
\crefalias{section}{appendix}
\crefalias{subsection}{appendix}

\section{Proof of \Cref{lem-obj-equiv}}\label{sec-lem-obj-equiv-proof} 

By definition,
\begin{align*}
-\log \PP\left( z= d \mid z>0 \right)
&=
-\log \frac{\PP(z=d)}{1-\PP(z=0)} \\[4pt]
&=
\log \PP(z=d) + \log \left( 1 - \PP(z=0) \right) \\[4pt]
&=
-\log \left[ \binom{N}{d} q^d (1-q)^{N-d} \right] + \log \left[ 1 - (1-q)^N \right] \\[4pt]
&=
-\log \binom{N}{d} - d\log q - (N-d) \log (1-q) + \log \left[ 1 - (1-q)^N \right].
\end{align*}
Since $-\log \binom{N}{d}$ is a constant independent of $q$, then it suffices to prove that the function
\begin{align*}
L_d(q) &=
- d\log q - (N-d) \log (1-q) + \log \left[ 1 - (1-q)^N \right] \\[4pt]
&= 
- (d-1)\log q - (N-d) \log (1-q) + \log \left[ \sum_{r=0}^{N-1}(1-q)^r \right]
\end{align*}
is convex on $(0,1/2]$. For every $d \ge 1$,
\[
L_d(q) = L_1(q) + (d-1)\log\frac{1-q}{q}.
\]
Since $q\mapsto \log\frac{1-q}{q}$ is convex on $(0,1/2]$, then it suffices to prove that the function $L_1(q)$ is convex.

We now prove that the function $L_1(q)$ is in fact convex on $(0,1)$. Let $A(t) = \sum_{r=0}^{N-1}t^r$, then $A'(t) = \sum_{r=1}^{N-1} r t^{r-1}$ and $A''(t) = \sum_{r=2}^{N-1} r(r-1) t^{r-2}$. We have
\begin{align}
& L_1'(q) = \frac{N-1}{1-q} - \frac{A'(1-q)}{A(1-q)}, \notag \\[4pt]
& L_1''(q) = \frac{N-1}{(1-q)^2} + \frac{A''(1-q)}{A(1-q)} - \left( \frac{A'(1-q)}{A(1-q)} \right)^2. \label{eqn-trunc-loss-Hessian-1}
\end{align}
Define a discrete random variable $R$ with distribution
\[
\PP(R = r) = \frac{(1-q)^r}{\sum_{s=0}^{N-1}(1-q)^s}, \quad \forall r\in\{0,1,...,N-1\}.
\]
We observe that
\begin{align}
& \EE \left[ R \right] = \frac{\sum_{r=0}^{N-1}r(1-q)^r}{\sum_{r=0}^{N-1}(1-q)^r} = \frac{A'(1-q)}{A(1-q)} \cdot (1-q), \label{eqn-trunc-loss-rv-1} \\[4pt]
& \EE\left[ R(R-1) \right] = \frac{\sum_{r=0}^{N-1}r(r-1)(1-q)^r}{\sum_{r=0}^{N-1}(1-q)^r} = \frac{A''(1-q)}{A(1-q)} \cdot (1-q)^2. \label{eqn-trunc-loss-rv-2}
\end{align}
Substituting \eqref{eqn-trunc-loss-rv-1} and \eqref{eqn-trunc-loss-rv-2} into \eqref{eqn-trunc-loss-Hessian-1} yields that for all $q\in(0,1)$,
\begin{align*}
L_1''(q) 
&= 
\frac{N-1}{(1-q)^2}
+
\frac{\EE[R^2] - \EE [R]}{(1-q)^2} - \frac{\EE[R]^2}{(1-q)^2} \\[4pt]
&=
\frac{N-1 + \var(R) - \EE[R]}{(1-q)^2} > 0
\end{align*}
because $\EE[R] < N-1$. Thus, $L_1(q)$ is strongly convex on $(0,1)$. This completes the proof.

\newpage

{
\bibliographystyle{ims}
\bibliography{bib}
}

\end{document}